\documentclass[twoside]{article}

\usepackage[accepted]{aistats2024}

\usepackage[round]{natbib}

\usepackage{amsfonts}
\usepackage{amsthm}
\usepackage{amssymb}
\usepackage{bm}
\usepackage{hyperref}
\usepackage{cleveref}
\usepackage{algorithm}
\usepackage{algpseudocode}
\usepackage[algo2e,ruled,linesnumbered]{algorithm2e}
\setRightLinesNumbers
\usepackage{etoolbox}

\makeatletter
\patchcmd{\@algocf@start}
  {-1.5em}
  {0pt}
  {}{}
\makeatother

\DeclareMathOperator*{\Exp}{\mathbb{E}}

\newcommand{\mur}[1]{{\bm{\mu} }_T^{[#1]}}
\newcommand{\emur}[1]{{\bm{\hat{\mu}}_T^{[#1]} }}

\newtheorem{thm}{Theorem}[section]
\newtheorem{prop}{Proposition}[section]

\begin{document}

%

%

\twocolumn[

\aistatstitle{An Improved Algorithm for Learning 
Drifting Discrete Distributions}

\aistatsauthor{ Alessio Mazzetto  }

\aistatsaddress{ Brown University }]

\begin{abstract}

We present a new adaptive algorithm for learning discrete distributions under distribution drift. In this setting, we observe a sequence of independent samples from a discrete distribution that is changing over time, and the goal is to estimate the current distribution. Since we have access to only a single sample for each time step, a good estimation requires a careful choice of the number of past samples to use. To use more samples, we must resort to samples further in the past, and we incur a drift error due to the bias introduced by the change in distribution. On the other hand, if we use a small number of past samples, we incur a large statistical error as the estimation has a high variance.  We present a novel adaptive algorithm that can solve this trade-off without any prior knowledge of the drift. Unlike previous adaptive results, our algorithm characterizes the statistical error using data-dependent bounds. This technicality enables us to overcome the limitations of the previous work that require a fixed finite support whose size is known in advance and that cannot change over time. Additionally, we can obtain tighter bounds depending on the complexity of the drifting distribution, and also consider distributions with infinite support.

\end{abstract}

\section{INTRODUCTION}

Estimating a distribution from a set of samples is a crucial challenge in data analysis and statistics \citep{Devroye1987NonparametricDE, silverman1986density, devroye2001combinatorial}. In this work, we focus on the classical setting of estimating the probability mass function of a discrete distribution.  A long list of work characterized the error for this estimation problem given \emph{independent} and \emph{identically distributed} (i.i.d.) samples from the same distribution \citep{han2015minimax, kamath2015learning, orlitsky2015competitive, jiao2015minimax, cohen2020learning}. The use of the total variation metric as a measure of error for this problem is a natural choice that is commonly adopted in the literature \citep{devroye2001combinatorial}. It is folklore that if a distribution has support size $k$, the maximum likelihood estimator with $n$ i.i.d. samples has an expected error upper bounded by $O( \sqrt{k/n})$, which can be shown to be tight \citep{anthony1999neural}.

In numerous applications where samples are collected over time, it is possible that their underlying distribution may change.  In the \emph{distribution drift} setting, we are interested in estimating the current distribution, given a sequence of past samples that could be generated by different distributions. This setting has been recently studied by \cite{mazzetto2023nonparametric}, but other problems have also been considered in a similar setting, for example, binary classification \citep[e.g.,][and references therein]{barve1996complexity, long1998complexity}, agnostic learning \citep{mohri2012new,hanneke2019statistical,mazzetto2023adaptive}, or crowdsourcing \citep{fu2020fast}.

Concretely, let $X_1,\ldots,X_T$ denote a sequence of $T$ \emph{independent} samples respectively from the discrete distributions $\bm{\mu}_1,\ldots,\bm{\mu}_T$. Equivalently, we say that the sequence of samples is generated over time by a \emph{drifting} discrete distribution. The goal is to estimate the current distribution $\bm{\mu}_T$ given the sequence of samples. Without loss of generality, it is sufficient to consider discrete distributions over the natural numbers, and given such a distribution $\bm{\mu}$, we let $\bm{\mu}(i) = \Pr_{\bm{\mu}}(X=i)$ for any $i \in \mathbb{N}$. The total variation distance between two discrete distributions $\bm{\mu}$ and $\bm{\eta}$ is defined as $\lVert \bm{\mu} - \bm{\eta} \rVert_{\mathrm{TV}} = (1/2)\sum_{i \in \mathbb{N}}| \bm{\mu}(i) - \bm{\eta}(i)|$.

Indeed, the estimation of $\bm{\mu}_T$ is possible only if the previous distributions are related to it. Let $\Delta_1,\ldots,\Delta_T$ be a non-decreasing sequence of real numbers where
\begin{align}
\label{eq:delta-definition}
\Delta_r \doteq \max_{t: 0 \leq t < r}\lVert \bm{\mu}_T - \bm{\mu}_{T-t} \rVert_{\mathrm{TV}} \enspace .
\end{align}
The value $\Delta_r$ is the maximum total variation distance from the current distribution $\bm{\mu}_T$ to any of the most recent $r$ distributions $\bm{\mu}_{T-r+1}, \ldots, \bm{\mu}_T$. In past work \citep{mazzetto2023nonparametric}, it is shown that if the distributions have the same support of size $k$, a tight lower bound on the expected error of \emph{any} algorithm for the estimation of the current distribution  $\bm{\mu}_T$ with respect to the total variation distance is given by
\begin{align}
\label{minimax-error}
\Omega \left( \min_{1 \leq r \leq T} \left[ \sqrt{\frac{k}{r}} + \Delta_r \right] \right) \enspace .
\end{align}
This lower bound is in a minimax sense. Formally, given any (possibly adaptive) algorithm, and a sufficiently small non-decreasing sequence of values $\Delta_1,\ldots,\Delta_T$, there exists a sequence of distributions $\bm{\mu}_1,\ldots,\bm{\mu}_T$ with shared support $k$ such that $\max_{t: 0 \leq t < r}\lVert \bm{\mu}_{T-t} -  \bm{\mu}_T\rVert_{\mathrm{TV}} \leq \Delta_{r}$ for any $1 \leq r \leq T$, and the expected estimation error of the algorithm is at least $\eqref{minimax-error}$. 

For a fixed integer $r$, the quantity $O(\sqrt{k/r} + \Delta_{r})$ of \eqref{minimax-error} is also an upper bound to the expected error obtained by only using the most recent $r$ samples for the estimation, and it is written as the sum of the upper bound to two errors: the \emph{statistical error} and the \emph{drift error}. The statistical error term $\sqrt{k/r}$ is related to the variance of the estimation, and it decreases by considering more samples; whereas the drift error term $\Delta_{r}$ is due to the distribution drift, and it can potentially increase by using samples further away in time. Equation \eqref{minimax-error} shows that an optimal estimation is given by the optimal solution of this trade-off. This trade-off determines an optimal number of recent samples to use as a function of the distribution drift. This is a significant difference with the i.i.d. setting, where each sample provides useful information, and the expected error goes to $0$ as the number of samples goes to infinity. 

The trade-off between the statistical error and the drift error is common in the literature on learning with drift \citep{mohri2012new,mazzetto2023nonparametric}. However, the minimization of this trade-off is challenging as the values $\Delta_1,\ldots,\Delta_T$ of the drift error are unknown, and they cannot be estimated from the data since we only have access to a single sample from each distribution. For this reason, most of the previous work required prior knowledge of the magnitude of the drift in order to quantify and solve this trade-off. For example, a common assumption is that the magnitude of the drift is bounded by $\Delta > 0$ at each step \citep{bartlett1992learning}, which implies  $\Delta_r \leq (r-1)\Delta$ for all 
$r \leq T$. In this case, the optimal solution of the trade-off gives an estimation error equal to $\Theta( (k \cdot \Delta)^{1/3})$ which is achieved by computing the empirical distribution induced by the most recent $\Theta( (k/\Delta^{2})^{1/3})$ samples. 

\subsection{Limitations of Existing Work}

In recent work, \cite{mazzetto2023adaptive} exhibit a general learning algorithm that solves the trade-off between statistical error and drift error (up to loglog factors) based on the input samples and without any prior knowledge of the drift. In our setting, this implies that there exists an algorithm that can adaptively attain the lower bound \eqref{minimax-error}, which cannot be improved in a minimax sense. Precisely, there exists an algorithm that observes the sequence $X_1,\ldots, X_T$ from a drifting distribution with a fixed support of size $k$, and it returns an estimate $\bm{\hat{\mu}}$ of $\bm{\mu}_T$ such that with probability at least $0.99$:
\begin{align}
\label{old-result-adaptive}
\lVert \bm{\mu}_T - \bm{\hat{\mu}} \rVert_{\mathrm{TV}} = O\left( \min_{1 \leq r \leq T}\left[ \sqrt{\frac{k}{r}} +  \sqrt{\frac{\log \log r}{r} } + \Delta_r \right] \right)
\end{align}
While the above error is essentially tight according to \eqref{minimax-error}, this result has several intertwined weaknesses.

First, the previous adaptive algorithm can only be applied to drifting distributions that have finite support, which cannot change over time. This constraint is in conflict with the drift setting, where the distribution, hence its support, can indeed change over time. For example, consider the drifting distribution of the items purchased over time from an online retailer: the support of this distribution can repeatedly evolve due to changes in inventory, new products, or availability. Additionally, we are also required to know the value $k$ of the support size. Since $k$ cannot be determined precisely using only the input samples, it is necessary to have prior knowledge of this value, which can be unfeasible in many practical applications.

Second, the aforementioned algorithm has the crucial shortcoming of using a distribution-independent upper bound $O(\sqrt{k/r})$ on the statistical error from using $r$ samples. While this upper bound is indeed tight for distributions that are roughly uniform over a support of size $k$, the actual error due to the variance of the estimation can be significantly smaller for other distributions. As an example, if we consider a distribution of size $k$, where most of the probability mass is concentrated in $k' \ll k$ elements, we would expect a statistical error with rate $O(\sqrt{k'/r})$. Furthermore, the use of a distribution-independent upper bound on the statistical error prevents the consideration of distributions with infinite support ($k = \infty$).
In the i.i.d. setting, it is possible to address this issue by using a sharp distribution-dependent upper bound on the statistical error, which also allows us to handle distributions with infinite support \citep{berend2013sharp,cohen2020learning}. In particular, if we assume that there is no drift, i.e. $\bm{\mu} = \bm{\mu}_1 = \ldots = \bm{\mu}_T$, the maximum likelihood  estimator $\bm{\hat{\mu}}$ over $T$ samples exhibits an expected error \citep{berend2013sharp}:
\begin{align}
\label{eq:iid-real}
    \frac{1}{8}\Lambda_{T}(\bm{\mu}) - 
 \frac{1}{4\sqrt{T}} \leq \Exp \lVert \bm{\mu} - \bm{\hat{\mu}} \rVert_{\mathrm{TV}} \leq \Lambda_{T}(\bm{\mu}) \enspace ,
\end{align}
where
\begin{align}
\label{eq:complexity-discrete}
    \Lambda_T(\bm{\mu}) \doteq \sum_{i : \bm{\mu}(i) < 1/T} \bm{\mu}(i) + \frac{1}{\sqrt{T}}\sum_{i : \bm{\mu}(i) \geq  1/T}\sqrt{\bm{\mu}(i)} \hspace{1pt} .
\end{align}
The value $\Lambda_{T}(\bm{\mu})$ is a measure of the \emph{learning complexity} of $\bm{\mu}$ that provides a tight characterization of the variance of the estimation of $\bm{\mu}$ with $T$ samples. By using the Cauchy-Schwarz inequality, it is simple to verify that if the support of $\bm{\mu}$ has size $k$, it holds that $\Lambda_{r}(\bm{\mu}) \leq \sqrt{k/r}$, recovering the aforementioned distribution-independent upper bound.
We highlight that having a tight bound on the statistical error is especially important in a drift setting as it can significantly impact the quality of the estimation, since the value of this bound determines the number of samples to use in order to solve the trade-off between statistical error and drift error. 

\section{MAIN RESULT}
In our work, we address the issues outlined in the previous section. Our main contribution is to provide an adaptive algorithm for estimating an \emph{arbitrary} drifting discrete distribution. The result is formalized as follows.
\begin{thm}
\label{thm:main-result}
Let $\delta \in (0,1)$. There exists an algorithm that given $X_1,\ldots,X_T$, it outputs a distribution $\bm{\hat{\mu}}$ such that with probability at least $1-\delta$, it holds that
\begin{align*}
 \lVert \bm{\mu}_T - \bm{\hat{\mu}} \rVert_{\mathrm{TV}}  = O\bigg(  \min_{1 \leq r \leq T} \bigg[\Lambda_{r}(\bm{\mu}_T)    \\ +\sqrt{\frac{\log((\log^2r+1)/\delta)}{r}}  + 
 \Delta_r \bigg]\bigg) \enspace,
\end{align*}
where $\Delta_r = \max_{0 \leq t < r}\lVert \bm{\mu}_T - \bm{\mu}_{T-t} \rVert_{\mathrm{TV}}$ as in \eqref{eq:delta-definition}.
\end{thm}
The above theorem shows that there exists an adaptive algorithm that can achieve (up to loglog factors) an optimal solution of the trade-off between statistical error and drift error, where the statistical error is quantified using a distribution-dependent measure of complexity. Compared to the previous adaptive result \eqref{old-result-adaptive}, our algorithm works for an \emph{arbitrary} discrete drifting distribution, and it utilizes a sharper distribution-dependent upper bound on the statistical error that we estimate from the data. In particular, our algorithm does not require any prior knowledge of the drifting distribution, and it also works for drifting distribution with support that changes over time or with infinite support.  For the special case of drifting distributions with shared support of size $k$, it holds that $\Lambda_{r}(\bm{\mu}_T) \leq \sqrt{k/r}$, and we indeed achieve the lower bound \eqref{minimax-error} up to logarithmic terms.
However, we highlight that our algorithm uses a distribution-dependent measure $\Lambda_{r}(\bm{\mu}_T)$ to upper bound the statistical error of the estimation which could be significantly tighter than $\sqrt{k/r}$ even for distributions with support size $k$. This enables us to obtain tighter bounds depending on the drifting distribution's complexity.

    Since $\Lambda_{r}(\bm{\mu}) \leq \Lambda_{s}(\bm{\mu})$ for any $s < r$, in the i.i.d. case, we can observe that the theorem guarantees with high-probability (e.g., $\geq 99/100$) an estimation such that
\begin{align*}
    \lVert \bm{\mu}_T - \bm{\hat{\mu}} \rVert_{\mathrm{TV}} = O\left( \Lambda_{T}(\bm{\mu}_T) + \sqrt{\frac{\log\log^2 T}{T}} \right) \enspace ,
\end{align*}
retrieving up to logarithmic terms the tight characterization of the estimation error depicted in \eqref{eq:iid-real} for learning with $T$ i.i.d. samples from $\bm{\mu}_T$. The additional logarithmic term is the cost of the adaptivity, since the algorithm does not know a priori whether the samples are identically distributed.


\textbf{Technical contribution.} It is not straightforward to extend the adaptive strategy of the previous work \citep{mazzetto2023adaptive} to use a data-dependent bound on the statistical error. In fact, the proof strategy of that work relies on knowing the exact rate at which the upper bound on the statistical error decreases, which is not possible for distribution-dependent upper bounds. To circumvent this issue, we develop a new analysis for learning with drift, whose proof of correctness also uses a novel result that ties the magnitude of the drift with the change in the learning complexity of the drifting distribution. We believe that our novel proof strategy for learning with drift using data-dependent upper bounds on the statistical error can also be applied to other learning problems. To the best of our knowledge, this is the first adaptive learning algorithm for discrete distributions to use data-dependent bounds in the drift setting.

\section{ALGORITHM}

In this section, we present the algorithm that achieves the guarantee of \Cref{thm:main-result}. First, we formally define the trade-off between statistical error and drift error obtained by using the most recent $r$ samples. To this end, for any $1 \leq r \leq T$, we define the following distributions
\begin{align*}
    \mur{r}(i) &= \frac{1}{r} \sum_{t = T-r+1}^T \bm{\mu}_t(i)   &\forall i \in \mathbb{N}& \enspace , \\ \emur{r}(i) &= \frac{1}{r} \sum_{t = T-r+1}^T \mathbf{1}_{\{X_t = i\}}  &\forall i \in \mathbb{N}& \enspace, 
\end{align*}
which are respectively the average distribution $\mur{r}$ and the empirical distribution $\emur{r}$ over the most recent $r$ samples. 
Following the methodology of previous work, the following proposition provides an error decomposition into statistical error and drift error for the estimation of $\bm{\mu}_T$ by using the empirical distribution $\emur{r}$.
\begin{prop}
\label{prop:error-decomposition}
Let $1 \leq r \leq T$. We have that
\begin{align*}
\hspace{-2pt}\lVert \bm{\mu}_T - \emur{r} \rVert_{\mathrm{TV}} \leq \underbrace{\lVert \mur{r} - \emur{r} \rVert_{\mathrm{TV}}}_{\text{ \hspace{-5pt} \normalfont Statistical Error}} \hspace{5pt}+\hspace{5pt}  \underbrace{\Delta_r}_{\text{ \hspace{-5pt} \normalfont Drift Error}} \enspace .
\end{align*}
\end{prop}
\begin{proof}
By using the triangle inequality, we have that
\begin{align*}
\lVert \bm{\mu}_T - \emur{r} \rVert_{\mathrm{TV}} \leq {\lVert \mur{r} - \emur{r} \rVert_{\mathrm{TV}}} +  {\lVert \bm{\mu}_T - \mur{r} \rVert_{\mathrm{TV}}}
\end{align*}
We can upper bound the second addend of the right-hand side above using the triangle inequality
\begin{align}
\label{eq:chain-inequality-drift}
    \lVert  \mur{r} - \bm{\mu}_T  \rVert_{\mathrm{TV}} 
    &= \left\lVert \frac{1}{r}\sum_{t=T-r+1}^T( \bm{\mu}_{t} - \bm{\mu}_T )\right\rVert_{\mathrm{TV}} \nonumber  \\
& \hspace{-20pt}  \leq  \hspace{5pt}  \frac{1}{r}\sum_{t=T-r+1}^T \lVert \bm{\mu}_t - \bm{\mu}_T \rVert_{\mathrm{TV}} 
\leq \Delta_r \enspace .
\end{align}
\end{proof}
The first term of the right-hand side of this proposition represents the statistical error of the estimation. Our goal is to measure this error as a function of the distribution-dependent measure of complexity $\Lambda_r(\bm{\mu}_T)$ of the current distribution $\bm{\mu}_T$, which is unknown. Our algorithm relies on the input data to estimate this quantity. Following the example of previous work, we use the empirical counterpart of this measure of complexity. Given an empirical distribution $\emur{r}$, we define 
\begin{align*}
    \Phi_r(\emur{r}) \doteq \sqrt{\frac{\lVert \emur{r} \rVert_{\frac{1}{2}}}{r}} = \frac{1}{\sqrt{r}} \sum_{i \in \mathbb{N}} \sqrt{\emur{r}(i)} \enspace ,
\end{align*}
and we observe that this quantity can be computed from the samples. The quantity $\Phi_r(\emur{r})$ is an empirical measure of complexity that provides an upper 
bound to the statistical error with $r$ samples. The next proposition is proven based on results of previous work \citep{cohen2020learning}.
\begin{prop}
\label{ub:statistical-error-empirical}
Let $\delta \in (0,1)$ and $1 \leq r \leq T$. With probability at least $1 - \delta$, it both holds:
\begin{align*}
\lVert  \emur{r} - \mur{r}\rVert_{\mathrm{TV}} \leq \Phi_r\left( \emur{r} \right) + 3 \sqrt{ \frac{\log(4/\delta)}{2r}} \enspace 
\end{align*}
and
\begin{align*}
\Phi_r\left( \emur{r} \right) \leq 4\Lambda_r\left( \mur{r} \right) + \sqrt{ \frac{\log(4/\delta)}{r}} 
\end{align*}
\end{prop}
\begin{proof}
In the appendix.
\end{proof}
This proposition enables us to use the empirical measure of complexity $\Phi_r\left( \emur{r} \right)$ to estimate the statistical error while guaranteeing that we obtain a result that is a tight approximation to using the non-empirical measure $\Lambda_r\left( \mur{r} \right)$ as in \eqref{eq:iid-real}. \Cref{ub:statistical-error-empirical} can be seen as a generalization of the results of \citet[Theorem~2.1 and Theorem~2.3]{cohen2020learning} for independent but not identically distributed discrete distributions.

On a high level, the core of our algorithm is a condition that allows us to compare the upper bound to the estimation error induced by different choices of the number of past samples \emph{without} explicitly estimating the drift error. For ease of notation, we let $r_j \doteq 2^j$ for any $j \geq 0$.
As long as this condition is true, the algorithm iteratively considers a larger number of past samples, also referred to as \emph{window size}, starting from $r_0=1$. In particular, at iteration $j$ it considers a window size $r_j$ (starting from $j=0$), and it evaluates the condition by comparing the estimation obtained with $r_j$ to the estimation obtained with $r_0, \ldots r_{j-1}$. If the condition is satisfied, we can provably maintain a solution whose upper bound to the estimation error is up to a constant factor as good as any previously considered window size (Proposition~\ref{prop:continue-condition}). Conversely, if the condition is violated, a non-negligible drift has occurred, thus using more samples cannot yield a significantly better estimation (Proposition~\ref{prop:stop-condition}). As we will see, depending on the data, it is possible that we do not need to consider all the possible window sizes $\{ r_j : j \geq 0 \}$ for the evaluation of this condition, and only a subset will suffice. 

\begin{algorithm2e}
\LinesNumbered
\label{algo:adaptive-algorithm}
\caption{Adaptive Learning Algorithm For A Discrete Drifting Distribution }
$L = \{ 0 \}$ \\
\For{$j=1,\ldots,\lfloor \log_2 T\rfloor$}{
    \If{$\xi_{r_j} < \min_{\ell \in L}\xi_{r_\ell}$}{
        \For{$\ell \in L$}{
            \If{$\lVert \emur{r_{\ell}}- \emur{r_{j}} \rVert_{\mathrm{TV}} \geq 3\xi_{r_\ell} +\xi_{r_{j}}$}{
                \Return $\emur{r_{\max L}}$
            }
        }
        $L \gets L \cup \{ j \}$
    }
}
\Return $\emur{r_{\max L}}$
\end{algorithm2e}

The correctness of our algorithm is conditioned on the event that the estimation for all window sizes $r_j = 2^j$ for $j \geq 0$ concentrates around its expectation, i.e. we want to provide an upper bound to the statistical error as in \Cref{ub:statistical-error-empirical} for all those window sizes. This result is formalized in the next proposition and it is obtained by simply taking a union bound.

\begin{prop}
\label{prop:well-estimated-all}
Let $\delta \in (0,1)$.  With probability at least $1-\delta$, for all $j \geq 0$ it holds that 
\begin{align}
\label{eq:statistical-error-empirical}
\lVert  \emur{r_j} - \mur{r_j}\rVert_{\mathrm{TV}} \leq \Phi_{r_j}\left( \emur{r_j} \right) + 3 \sqrt{ \frac{\log\left(\frac{c(\log^2 r_j+1)}{\delta}\right)}{r_j}} \enspace 
\end{align}
and
\begin{align}
\label{eq:statistical-error-true}
\Phi_{r_j}\left( \emur{r_j} \right) \leq 4\Lambda_{r_j}\left( \mur{r_j} \right) +3 \sqrt{ \frac{\log\left(\frac{c(\log^2 r_j+1)}{\delta}\right)}{r_j}}
\end{align}
where $c$ is a constant equal to $c = 4 \pi^2/3$.
\end{prop}
\begin{proof}
Let $\delta_j = \delta (6/\pi^2)/(j+1)^2$. We have that with probability at least $1-\delta_j$, the event of \Cref{ub:statistical-error-empirical} holds with error probability $\delta_j$ and window size $r_j$. The statement follows by substituting the definition of $\delta_j$, and taking a union bound over all possible events for $j \geq 0$, since $\sum_{j \geq 0} \delta_j = (6/\pi^2)\delta\sum_{j \geq 1} 1/j^2 = \delta$.
\end{proof} 
The value $\delta \in (0,1)$ of the above proposition is a parameter of the algorithm, and it denotes its failure probability. Throughout this section, we assume that the event of \Cref{prop:well-estimated-all} holds, otherwise our algorithm fails (with probability $\leq \delta$). We denote with 
\begin{align*}
\xi_{r_j} \doteq \Phi_{r_j}\left( \emur{r_j} \right) +  3 \sqrt{ \frac{\log(c(\log^2 r_j+1)/\delta)}{r_j}} \enspace 
\end{align*}
the upper bound to the statistical error for the window size $r_j$ given by the empirical bound \eqref{eq:statistical-error-empirical} in \Cref{prop:well-estimated-all}.
By using \Cref{prop:error-decomposition}, we have the following upper bound to the error of estimating $\bm{\mu}_T$ by using the empirical distribution $\emur{r_j}$:
\begin{align}
\label{ub-xi}
    U(r_j) &\doteq \xi_{r_j} + \Delta_{r_j} \geq 
     \lVert \bm{\mu}_T - \emur{r_j} \rVert_{\mathrm{TV}} ,\hspace{10pt}  \forall j \geq 0 .
\end{align}

The algorithm is looking for the window size $r_j$ that minimizes the upper bound $U(r_j)$. We remark that this is a challenging problem due to the fact that the drift error is unknown and it cannot be estimated from the data, since we have access to a single sample from each distribution. Additionally, the sequence $( \xi_j)_{j \geq 0}$ is not necessarily strictly decreasing, and it does not have an analytical closed formula that depends only on $r_j$. This is a requirement of the proof strategy of the previous adaptive work \cite{mazzetto2023adaptive}, which exploits the a priori knowledge of how much the statistical error is reduced by doubling the window size. Therefore, their proof strategy does not apply to our setting.

By an inspection of $U(\cdot)$, it is clear that if $\xi_{j+n} > \xi_j$ for some $n \geq 1$, then $U(r_j) < U(r_{j+n})$ since the drift error is non decreasing with respect to the window size. Thus, we can only consider a sequence of window sizes for which the upper bound to the statistical error is decreasing. Formally, we consider the following sequence of indexes $L = (\ell_1,\ldots,\ell_K)$ where $K \leq 1 + \log_{2}(T)$ that is built iteratively as follows. We let $\ell_1 = 0$. Given $\ell_1,\ldots,\ell_i$, the element $\ell_{i+1}$ is added (if it exists) as the first index $j > \ell_i$ such that $\xi_{j} < \xi_{\ell_{i}}$, thus we have $\xi_{\ell_{i+1}} < \xi_{\ell_{i}}$ for any $1 \leq i \leq K-1$. Moreover, given $j \geq 0$, we let $\gamma(j) = \ell$ be the largest value $\ell \in L$ such that $\ell \leq j$. Observe that by construction, the following relation holds:
\begin{align}
\label{eq:L-property}
U(r_{\gamma(j)}) \leq U(r_{j})    \enspace .
\end{align}
Therefore, we have $
    \min_{\ell \in L} U(r_{\ell}) \leq \min_{1 \leq j \leq T}U(r_j) $,
and we can only consider window sizes $r_{\ell}$ with values $\ell \in L$.

The pseudo-code of the algorithm is reported in Algorithm~\ref{algo:adaptive-algorithm}. The algorithm iteratively builds the list $L$. Once a new element $j$ that belongs to this list is found (Line~3), the algorithm compares the empirical distribution $\emur{r_j}$ with all the empirical distributions $\emur{r_\ell}$ with $\ell \in L$ such that $\ell < j$. This comparison provides the iteration condition that is the core of our algorithm. The statistical  error $\xi_{r_j}$ using $r_j$ samples is less than $\xi_{r_{\ell}}$ for all such $\ell$, however the drift error could be larger. The main idea is that if all those empirical distributions are sufficiently close (Line~5), then their distance with respect to $\bm{\mu}_T$ cannot be that large, and we can guarantee that the estimation error obtained by using $\emur{r_j}$ is as good as the estimation achieved with any $\emur{r_\ell}$ for all such $\ell$. In this case, we can keep iterating. This intuition is formalized with the following proposition.
\begin{prop}
\label{prop:continue-condition}
Assume that the event of \Cref{prop:well-estimated-all} holds. Let $j \in L$. If 
\begin{align*}
\lVert \emur{r_{\ell}}- \emur{r_{j}} \rVert_{\mathrm{TV}} \leq 3\xi_{r_\ell} + \xi_{r_j}\hspace{10pt} \forall \ell < j : \ell \in L
\end{align*}
then, we have that:
\begin{align*}
    \lVert \bm{\mu}_T - \emur{r_{j}} \rVert_{\mathrm{TV}} \leq 5 \cdot \min_{\substack{\ell \in L :\\ \ell < j}} U(r_\ell) \enspace .
\end{align*}
\begin{proof}
Let $\ell \in L$ such that $\ell < j$. By using the triangle inequality, we have that:
\begin{align*}
     \lVert \bm{\mu}_T - \emur{r_{j}} \rVert_{\mathrm{TV}} \leq   \lVert \bm{\mu}_T - \emur{r_{\ell}} \rVert_{\mathrm{TV}} +  \lVert  \emur{r_{\ell}} - \emur{r_{j}} \rVert_{\mathrm{TV}} \enspace .
\end{align*}
We upper bound the first term of the right-hand side using \eqref{ub-xi}, and the second term by using the assumption of this proposition. We obtain:
\begin{align*}
     \lVert \bm{\mu}_T - \emur{r_{j}} \rVert_{\mathrm{TV}} &\leq \xi_{r_\ell} + \Delta_{r_\ell} + \xi_{r_\ell} + 3\xi_{r_j} \\
     &\leq 5 \xi_{r_\ell} + \Delta_{r_{\ell}} \\
     &\leq 5 \cdot U(r_\ell) \enspace .
\end{align*}
\end{proof}
\end{prop}
Conversely, we want to show that if one of the conditions in Line~5 of the algorithm is violated, then we can stop iterating. If there exists a distribution $ \emur{r_\ell}$ that is far enough from $\emur{r_j}$ for a $ \ell \in L$ such that $\ell < j$, then a significant distribution drift must have occurred. In particular, we can show a lower bound to $\Delta_{r_j} \geq \xi_{r_\ell}$, and thus $U(r_n) \geq \xi_{r_\ell}$ for any $n \geq j$ due to the drift error. Since $U(r_{\ell}) = \xi_{r_\ell} +\Delta_{r_\ell}$, and the drift error from using $r_{\ell}$ samples is less or equal to the drift error from using $r_n$ samples, we are able to conclude that $U(r_{n})$ cannot be significantly smaller than $U(r_{\ell})$, and in particular $U(r_\ell) \leq 2U(r_n)$. This provides a certificate that we can stop iterating since the window size $r_{\ell}$ provides a value of the upper bound  $U(\cdot)$ that is up to constant as good as the one obtained with any window size $r_n$ with $n \geq j$. This result is formalized with the following proposition.

\begin{prop}
\label{prop:stop-condition}
Assume that the event of \Cref{prop:well-estimated-all} holds. Let $j \in L$. If there exists $\ell \in L$ with $\ell < j$ such that
\begin{align*}
\lVert \emur{r_{\ell}}- \emur{r_{j}} \rVert_{\mathrm{TV}} \geq 3\xi_{r_\ell} +\xi_{r_{j}} \enspace ,
\end{align*}
then $U(r_\ell) \leq 2U(r_{n})$ for any $n \geq j$.
\begin{proof}
By the triangle inequality, we have that
\begin{align}
& \lVert \emur{r_{\ell}}- \emur{r_{j}}\rVert_{\mathrm{TV}} \nonumber \\ \leq \hspace{10pt} &\lVert \emur{r_{\ell}}- \mur{r_{\ell}} \rVert_{\mathrm{TV}} + \lVert \emur{r_{j}}- \mur{r_{j}} \rVert_{\mathrm{TV}}  \nonumber \\  +&\lVert \mur{r_{\ell}}- \mur{r_{j}} \rVert_{\mathrm{TV}}  \nonumber \\
\leq \hspace{10pt} &  \xi_{r_j} + \xi_{r_\ell} + \lVert \mur{r_{\ell}}- \mur{r_{j}} \rVert_{\mathrm{TV}}  \enspace , \label{eq:tmp1}
\end{align}
where the last inequality follows from the assumption that the event of \Cref{prop:well-estimated-all} holds. Using the triangle inequality again, we obtain 
\begin{align*}
    \lVert \mur{r_{\ell}}- \mur{r_{j}} \rVert_{\mathrm{TV}} &\leq \lVert \mur{r_{\ell}}- \bm{\mu}_T \rVert_{\mathrm{TV}} +  \lVert \mur{r_{j}}- \bm{\mu}_T \rVert_{\mathrm{TV}} \\ 
    & \leq \Delta_{r_\ell} + \Delta_{r_j} \leq 2\Delta_{r_j} \enspace ,
\end{align*}
where in the last two inequalities we used relation \eqref{eq:chain-inequality-drift} and the fact that the drift error is non-decreasing with the number of past samples. By combining the above upper bound with \eqref{eq:tmp1}, we have
\begin{align*}
    \lVert \emur{r_{\ell}}- \emur{r_{j}}\rVert_{\mathrm{TV}}  \leq  \xi_{r_j} + \xi_{r_\ell} + 2\Delta_{r_j} \enspace .
\end{align*}
We use the assumption of the proposition and obtain the following lower bound to the drift error 
\begin{align*}
    \Delta_{r_j} \geq \xi_{r_\ell} \enspace .
\end{align*}
For any $n \geq j$, we have that
\begin{align*}
    2U(r_n) &= 2\xi_{r_n} + 2\Delta_{r_n} \geq   2\Delta_{r_j}  
    \geq \xi_{r_\ell} + \Delta_{r_\ell} = U(r_{\ell}) \enspace .
\end{align*}
\end{proof}
\end{prop}

\Cref{prop:continue-condition} and \Cref{prop:stop-condition} can be used to prove that our algorithm finds a window size $\hat{r} = 2^j$ for some $j \geq 0$ such that $U(\hat{r}) = \min_{i}U(r_i)$. We can express this upper bound using the measure of complexity $\Lambda_{\hat{r}}(\cdot)$ thanks to \eqref{eq:statistical-error-true}. However, this is not sufficient to prove \Cref{thm:main-result}, since we are only considering window sizes $r$ that are powers of two, and we want to compare against any possible selection of the window size $1 \leq r \leq T$. For window sizes that are not power of two, \Cref{prop:well-estimated-all} does not provide direct information on the statistical error. To prove the theorem, we need to relate the magnitude of the drift with the change in complexity of the drifting distribution. The following proposition provides this result.
\begin{prop}
\label{prop:relation-variance-drift}
Let $\bm{\mu}$ and $\bm{\eta}$ be two discrete distributions over $\mathbb{N}$. For any integers $1 \leq r \leq s$, the following two inequalities hold:
\begin{align*}
|\Lambda_{r}(\bm{\mu}) - \Lambda_{r}(\bm{\eta})| &\leq 2\lVert \bm{\mu} - \bm{\eta} \rVert_{\mathrm{TV}} \enspace , \\
\frac{\Lambda_r(\bm{\mu})}{\Lambda_s(\bm{\mu})} &\leq \sqrt{s/r} \enspace .
\end{align*}
\end{prop}
\begin{proof}
We start by proving the first inequality. We partition $\mathbb{N}$ into four sets:
$S_{00} = \{i : \bm{\mu}(i) < 1/r \land \bm{\eta}(i) < 1/r \}$, 
$S_{01} = \{i : \bm{\mu}(i) < 1/r \land \bm{\eta}(i) \geq 1/r \}$,
$S_{10} = \{i : \bm{\mu}(i) \geq 1/r \land \bm{\eta}(i) < 1/r \}$, and $ 
S_{11}= \{i : \bm{\mu}(i) \geq 1/r \land \bm{\eta}(i) \geq 1/r \}$. We have the following decomposition
\begin{align*}
&\Lambda_r(\bm{\mu}) - \Lambda_r(\bm{\eta}) =  \sum_{i \in S_{00}} \left(\bm{\mu}(i) - \bm{\eta}(i)\right) \\ 
&+\sum_{i \in S_{01}} \left(\bm{\mu}(i) - \sqrt{\frac{\bm{\eta}(i)}{r}}\right)+\sum_{i \in S_{10}} \left(\sqrt{\frac{\bm{\mu}(i)}{r}} - \bm{\eta}(i)\right)
\\&+\frac{1}{\sqrt{r}}\sum_{i \in S_{11}} \left(\sqrt{\bm{\mu}(i)} - \sqrt{\bm{\eta}(i)}\right) \enspace .
\end{align*}
The sum over $S_{00}$ is upper bounded simply as
\begin{align*}
    \sum_{i \in S_{00}} \left(\bm{\mu}(i) - \bm{\eta}(i)\right) \leq \sum_{i \in S_{00}} \left|\bm{\mu}(i) - \bm{\eta}(i) \right| \enspace . \\
\end{align*}
For the sum over $S_{11}$, we have that
\begin{align*}
     \sum_{i \in S_{11}} \frac{\left[\sqrt{\bm{\mu}(i)} - \sqrt{\bm{\eta}(i)}\right]}{\sqrt{r}} &= \frac{1}{\sqrt{r}} \sum_{i \in S_{11}}  \frac{\left(\bm{\mu}(i) - \bm{\eta}(i)\right)}{\sqrt{\bm{\mu}(i)}+\sqrt{\bm{\eta}(i)}} \\
    &\leq \frac{\sqrt{r}}{2\sqrt{r}}\sum_{i \in S_{11}} | \bm{\mu}(i) - \bm{\eta}(i)| \\
    &=\sum_{i \in S_{11}} \frac{| \bm{\mu}(i) - \bm{\eta}(i)|}{2} \enspace ,
\end{align*} 
where in the inequality we used the fact that $\sqrt{\bm{\mu}(i)}$ and $\sqrt{\bm{\eta}(i)}$ are both at least $\sqrt{1/r}$ for any $i \in S_{11}$.
Now, observe that for $i \in S_{01}$, we have that  $\bm{\mu}(i) = \sqrt{\bm{\mu}(i)}\cdot \sqrt{\bm{\mu}(i)} \leq \sqrt{\bm{\mu}(i)/r}$.
Using this inequality and proceeding similarly to the previous case, we have that
\begin{align*}
    \sum_{i \in S_{01}} \left(\bm{\mu}(i) - \sqrt{\frac{\bm{\eta}(i)}{r}}\right) &\leq \sum_{i \in S_{01}} \left(\sqrt{\frac{\bm{\mu}(i)}{r}} - \sqrt{\frac{\bm{\eta}(i)}{r}}\right) \\
    &\leq \sum_{i \in S_{01}} | \bm{\mu}(i) - \bm{\eta}(i)| \enspace .
\end{align*}
For $i \in S_{10}$, we observe that $\bm{\mu}(i) > 1/r$, hence $1/\sqrt{r} < \sqrt{\bm{\mu}(i)}$. Therefore, it holds that $\sqrt{\bm{\mu}(i)/r} < \sqrt{\bm{\mu}(i)} \cdot \sqrt{\bm{\mu}(i)} < \bm{\mu}(i)$, and
\begin{align*}
    \sum_{i \in S_{10}} \left(\sqrt{\frac{\bm{\mu}(i)}{r}} - \bm{\eta}(i)\right) \leq \sum_{i \in S_{10}} \left(\bm{\mu}(i) -  \bm{\eta}(i) \right) \enspace .
\end{align*}
We can conclude:
\begin{align*}
    \Lambda_r(\bm{\mu}) - \Lambda_r(\bm{\eta}) \leq \sum_{i \in \mathbb{N}}| \bm{\mu}(i) - \bm{\eta}(i)| \leq  2\lVert \bm{\mu} - \bm{\eta} \rVert_{\mathrm{TV}} \enspace .
\end{align*}

To prove the second inequality, we use instead the following partition of $\mathbb{N}$ into $S_{0} = \left\{ i: \bm{\mu}(i) < 1/s \right\}$, $S_{1} = \left\{ i: 1/s \leq \bm{\mu}(i) < 1/r \right\} $, and $
S_{2} = \{ i :  1/r \leq \bm{\mu}(i)  \}$. For any $i \in S_1$, we have that $\bm{\mu}(i) \leq  \sqrt{\bm{\mu}(i)/r} = \sqrt{s/r} \sqrt{\bm{\mu}(i)/s}$. We obtain the following result:
\begin{align*}
{\Lambda_r(\bm{\mu})} &=  \sum_{i \in S_0}\bm{\mu}(i) +  \sum_{i \in S_1 } \bm{\mu}(i) + \sum_{i \in S_2} \sqrt{\frac{\bm{\mu}(i)}{r}} \\
& \leq \sum_{i \in S_0}\bm{\mu}(i) + \sqrt{\frac{s}{r}}\sum_{i \in S_1} \sqrt{\frac{\bm{\mu}(i)}{s}} + \sqrt{\frac{s}{r}}\sum_{i \in S_2} \sqrt{\frac{\bm{\mu}(i)}{s}} \\
& \leq \sqrt{\frac{s}{r}} \Lambda_s(\bm{\mu}) \enspace .
\end{align*}
\end{proof}

We can finally prove \Cref{thm:main-result}.
\begin{proof}[Proof of~\Cref{thm:main-result}]
We assume that the event of \Cref{prop:well-estimated-all} holds, otherwise we say that our algorithm fails (with probability $\leq \delta$). The algorithm returns an empirical distribution $\emur{r_j}$ for some $j \geq 0$, and it guarantees that
\begin{align*}
    \lVert \bm{\mu}_T - \emur{r_j} \rVert_{\mathrm{TV}} \leq U(r_j)  \enspace .
\end{align*}
Consider the function $Q(r) : \{1,\ldots,T\} \mapsto \mathbb{R}$ defined as
\begin{align*}
    Q(r) \doteq
    \Lambda_r( \bm{\mu}_T) + \sqrt{\frac{\log(c(\log^2 r+1)/\delta)}{r}}  + \Delta_r \enspace ,
\end{align*}
where $c$ is the same constant in \Cref{prop:well-estimated-all}. Let $r^* = \mathrm{argmin}_{1 \leq r \leq T}Q(r)$. In order to prove the theorem, it is sufficient to show that $U(r_j) = O( Q(r^*))$. Let $\gamma$ be defined as in \eqref{eq:L-property}, and let $i \geq 0$ be such that $\gamma(r^*) = r_{i}$. We distinguish two cases: $(a)$ $r_i \leq r_{j}$ and $(b)$ $r_i > r_j$. For both cases, we will use the following result:
\begin{align}
\label{eq:relation-optimal}
 U(r_i) = O( Q(r^*)) \enspace .
\end{align}
Equation \eqref{eq:relation-optimal} is proven as follows.
Let $n$ be the largest integer such that $2^n \leq r^*$. By construction, we have $\gamma(r_n) = r_i$, thus inequality \eqref{eq:L-property} shows us that $U(r_i) \leq U(r_n)$. By definition,  $U(r_n) = \xi_{r_n} + \Delta_{r_n}$, and 
\begin{align}
\label{eq:123}
   \xi_{r_n} &= \Phi(\emur{r_n}) + 3  \sqrt{\frac{\log (c(\log^2 r_n +1)/\delta)}{r_n}} \nonumber \\
    & \leq 4 \Lambda_{r_n}( \mur{r_n}) + 12 \sqrt{\frac{\log (c(\log^2 r^* +1)/\delta)}{r^*}} \enspace , 
\end{align}
where we used \eqref{eq:statistical-error-true} and the fact that $r^* \leq 2 r_n$. Through, \Cref{prop:relation-variance-drift} we obtain the following upper bound
\begin{align*}
 &\Lambda_{r_n}( \mur{r_n}) \\ \leq \hspace{5pt}&  \Lambda_{r_n}( \bm{\mu}_T) + 2 \Delta_{r_n}\\ \leq  \hspace{5pt}& \sqrt{r^*/r_n} \cdot \Lambda_{r^*}( \bm{\mu}_T)+ 2 \Delta_{r_j} \\
  \leq \hspace{5pt}&  4 \Lambda_{r^*}( \bm{\mu}_T) + 2 \Delta_{r^*},
\end{align*}
where in the second inequality we also used \eqref{eq:chain-inequality-drift}, and in the last inequality we used the fact that the drift error is non-decreasing and that $r^* \leq 2 r_n$. If we combine the above inequality with $\eqref{eq:123}$, we have that
\begin{align*}
    U(r_i) \leq U(r_n) &\leq 16 \Lambda_{r^*}(\bm{\mu}_T) + 12 \sqrt{\frac{\log\left(\frac{c(\log^2 r^*+1)}{\delta}\right)}{r^*}} \\ &+ 9 \Delta_{r^*} = O( Q(r^*)) \enspace .
\end{align*}

Equipped with \eqref{eq:relation-optimal}, We will now show that $U(r_j) = O( Q(r^*))$ for both cases $(a)$ and $(b)$.

Case $(a)$. Since $i \in L$, we have $U(r_j) \leq 5 U(r_i)$ due to \Cref{prop:stop-condition}. We conclude by using equation \eqref{eq:relation-optimal}.

Case $(b)$. The algorithm chooses the window size $r_j < r_i$. This means that when the algorithm considers the next element $\ell \in L$ after $j$, where $\ell \leq i$, there exists an index $\ell' \in L$ with $\ell' \leq j$ such that the condition of Line~5 of the algorithm is satisfied. \Cref{prop:stop-condition} applies, and we have that $U(r_{\ell'}) \leq 2U(r_{n})$ for any $n \geq \ell > j$. By construction, this is also true for $n$ equal to $i$, and  $U(r_{\ell'})  \leq 2U(r_{i})$. On the other hand, since the algorithm returned $r_j$, it means that all the If conditions on Line~5 must have been satisfied when considering $j \in L$, and \Cref{prop:continue-condition} gives us $U(r_j) \leq 5 \min_{z \leq j : z \in L} U(r_z)$, and in particular $U(r_j) \leq U(r_{\ell'})$. Thus, $U(r_j) \leq 10 U(r_i)$. We obtain the statement by using equation \eqref{eq:relation-optimal}.
\end{proof}

\section{RELATED WORK}
The problem of learning with distribution drift was introduced in the context of binary classification \citep{helmbold1991tracking, bartlett1992learning, helmbold1994tracking}. This line of research led to the result that if any two consecutive distributions have a bounded $L_1$ distance $\Delta$, the expected error for learning a family of binary functions with VC dimension $\nu$ is $O( (\nu \Delta)^{1/3})$ \citep{long1998complexity}, and the upper bound is tight \citep{barve1996complexity}. Under mild assumptions, \cite{mohri2012new} extend the analysis of agnostic learning with drift to any family of functions. In particular, they provide an upper bound to the learning error for a given window size that uses the Rademacher complexity to quantify the statistical error, and a problem-dependent upper bound to the drift error called discrepancy that is based on previous work in domain adaptation \citep{mansour2009domain,ben2010theory}. Recent work relaxes the independence assumption and provides learning bounds under mixing condition \citep{hanneke2019statistical}. This previous work either requires a priori knowledge of the drift error to solve the trade-off between statistical error and drift error, or assumes that multiple samples can be obtained from each distribution to estimate the drift error \citep{mohri2012new, awasthi2023theory}. There are two work that provide an adaptive algorithm for learning a family of functions with drift: \cite{hanneke2015learning} address the realizable case, and \cite{mazzetto2023adaptive} address the agnostic case. 

Recent work characterizes the minimax error for the problem of learning discrete and continuous smooth distributions with distribution drift, but it assumes a prior knowledge of the drift to attain this error \citep{mazzetto2023nonparametric}. In a more specific setting, \cite{gokcesu2017online} provide an adaptive algorithm for learning a parametric family of exponential densities, where the parameters can slowly drift over time. 
We finally point out that several algorithms have been proposed for estimating a density in the online setting \citep{kristan2011multivariate, garcia2012online}, but they do not provide an analysis of the estimation error. 

\section{CONCLUSION}
We provide an adaptive algorithm for the problem of learning a drifting discrete distribution. Unlike previous work, our method solves this problem for \emph{any} drifting discrete distribution, and it does not require any prior assumption on the support of the distribution. Additionally, our algorithm utilizes the input data to estimate the statistical error, and it can provide a tighter bound than existing methods depending on the complexity of the drifting distribution.
To the best of our knowledge, this is the first adaptive method for learning a discrete distribution to use data-dependent bounds in a drift setting.

\subsubsection*{Acknowledgements}
The authors would like to thank Eli Upfal for the helpful discussions. This material is based on research sponsored by the National Science Foundation (NSF) under award IIS-1813444, and by a Kanellakis Fellowship.

\bibliographystyle{apalike}
\bibliography{file}

\appendix
\onecolumn

\section{DEFERRED PROOFS}

\begin{proof}[Proof of Proposition~\ref{ub:statistical-error-empirical}]
The proof of this proposition is based on previous work in the literature. The first inequality of the proposition immediately follows from the proof of \citet[Theorem~2.1]{cohen2020learning}. Specifically, with probability at least $1-\delta/2$, it holds:
\begin{align}
\label{eq:42-first}
    \lVert  \emur{r} - \mur{r}\rVert_{\mathrm{TV}} \leq \Phi_r\left( \emur{r} \right) + 3 \sqrt{ \frac{\log(4/\delta)}{2r}} \enspace .
\end{align}

The proof of the second inequality of this proposition proceeds as follows. By invoking Fubini's theorem, we have that
\begin{align*}
    \Exp \Phi( \emur{r}) = \frac{1}{\sqrt{r}}\Exp\left(  \sum_{i=1}^\infty \sqrt{ \emur{r}(i)} \right) = \frac{1}{\sqrt{r}}  \sum_{i=1}^\infty \Exp\left(\sqrt{ \emur{r}(i)} \right) = \frac{1}{r}  \sum_{i=1}^\infty \Exp\left(\sqrt{r \cdot \emur{r}(i)} \right) \enspace .
\end{align*}
Consider a value $i \in \mathbb{N}$, and let $C_i = r  \emur{r}(i)$. The crucial observation is that
\begin{align*}
\Exp C_i = r \Exp\left(\frac{1}{r}\sum_{t=T-r+1}^T \mathbf{1}_{ \{ X_t = i \}} \right) = \sum_{t=T-r+1}^T \Exp \mathbf{1}_{ \{ X_t = i \}}  = \sum_{t=T-r+1}^T \bm{\mu}_t(i) = r \mur{r}(i) \enspace .
\end{align*}

The remaining of the proof follows the same argument in previous work \citep{cohen2020learning}. Since the square root is a concave function, we have that $\Exp \sqrt{C_i} \leq \sqrt{\Exp C_i}$ by using Jensen's inequality. Furthermore, we can exploit that $\sqrt{C_i} \leq C_i$ as $C_i \in \{0,1,\ldots,r\}$, to show $\Exp \sqrt{C_i} \leq \Exp C_i$. We obtain:
\begin{align*}
 \Exp \Phi( \emur{r}) &= \frac{1}{r} \sum_{i=1}^\infty \Exp \sqrt{C_i} \\ &\leq \frac{1}{r} \sum_{i=1}^\infty  \min\left\{ r\mur{r}(i), \sqrt{r \mur{r}(i)}   \right\} \\ &= \sum_{i : \mur{r}(i) \leq 1/r} \mur{r}(i)  + \frac{1}{\sqrt{r}}\sum_{i : \mur{r}(i) > 1/r} \mur{r}(i) \\
 &= \Lambda_{r}(\mur{r}) \enspace .
\end{align*}
Since changing a single sample among $\{X_{T-r+1},\ldots,X_T\}$ can change the value of $\Phi( \emur{r})$ by at most $2/r$, we can use McDiarmid's inequality to show that with probability at least $1-\delta/2$, it holds that:
\begin{align}
\label{eq:42-second}
     \Phi( \emur{r}) \leq \Exp \Phi( \emur{r}) + \sqrt{\frac{\log(2/\delta)}{r}} \enspace .
\end{align}
We conclude by taking a union bound so that both events \eqref{eq:42-first} and \eqref{eq:42-second} hold with probability at least $1-\delta$.
\end{proof}

\end{document}